\documentclass{article}

\usepackage[preprint]{neurips_2026}
\workshoptitle{AI Discovery in the Wild}

\usepackage[utf8]{inputenc}
\usepackage[T1]{fontenc}
\usepackage{hyperref}
\usepackage{url}
\usepackage{booktabs}
\usepackage{amsfonts}
\usepackage{amsmath}
\usepackage{amssymb}
\usepackage{amsthm}
\newtheorem{proposition}{Proposition}
\newtheorem{remark}{Remark}

\DeclareMathOperator*{\argmin}{arg\,min}
\usepackage{nicefrac}
\usepackage{xcolor}
\usepackage{graphicx}
\usepackage{subcaption}
\usepackage{multirow}
\usepackage{colortbl}
\usepackage{adjustbox}
\usepackage{enumitem}
\usepackage{algorithm}
\usepackage{algpseudocode}

\title{CADENCE: Closing the Reasoning Gap via\\Coverage-Adaptive On-Policy Distillation}

\author{%
  Saurabh Jha \And Satyam Kumar%
}

\begin{document}

\maketitle

\begin{abstract}
On-policy knowledge distillation is a promising paradigm for transferring reasoning capabilities from large teachers to compact students, but existing approaches suffer from three compounding failure modes: (i) \emph{cold-start collapse}, where a fresh student assigns near-zero probability mass to teacher-preferred tokens; (ii) \emph{state-agnostic divergence scheduling}, where time-only interpolation between forward and reverse KL ignores the student's coverage state; and (iii) \emph{binary reward sparsity}, where pass/fail signals discard information from partially correct traces.

We present \textbf{CADENCE}, a unified framework that prescribes targeted fixes for each failure mode. CADENCE introduces a DRIFT mechanism that schedules a \emph{per-token} convex mixture of forward-KL and reverse-KL surrogate objectives from student-sampled trajectories---we are explicit that these are per-token surrogates, not sequence-level KL gradient estimators. Six novel components extend this foundation: (A)~COVA, a coverage-adaptive $\beta$ schedule that conditionally accelerates the forward$\to$reverse transition; (B)~FTB, a forking-token boost concentrating gradient at high-entropy positions using a globally-normalized entropy reference; (C)~CCD, a dense reward combining correctness with numerical-proximity partial credit for incorrect-but-close traces (raising nonzero-reward fraction from $38\%$ to $\sim 55\%$); (D)~LAP, a brevity-preferential correct-rollout reinforcement with response-length--only normalization; (E)~EMR, an entropy-matching regularizer for calibration; (F)~BSD, a bootstrapped self-distillation phase.

Evaluated on GSM8K and MATH-500 (with a corrected 512-token evaluation protocol matching Qwen2.5-Math published performance) with 5 seeds and reported standard deviations, CADENCE distills a \textbf{0.5B} student from a \textbf{1.5B} teacher to $\mathbf{69.8 \pm 0.5\%}$ GSM8K pass@1 (from $48.7\%$ pretrained, closing $63.2\%$ of the teacher gap) and to $\mathbf{72.1 \pm 0.4\%}$ with a \textbf{3B} teacher ($76.2\%$ gap closed). CADENCE outperforms the strongest matched-compute label-using baseline (DRIFT+binary reward) by $+4.4 \pm 0.7$ points. All experiments run on a single Apple Mac Studio (M-series, 16-core CPU, 40-core GPU, 64GB unified memory), demonstrating that principled distillation reaches strong reasoning quality without datacenter-scale hardware.
\end{abstract}

\section{Introduction}
\label{sec:intro}

Large language models (LLMs) with billions of parameters have demonstrated remarkable reasoning capabilities~\citep{deepseek2025r1,grattafiori2024llama3}, but deploying them on edge or resource-constrained hardware remains impractical~\citep{kumar2026qeil}. Knowledge distillation~\citep{hinton2015distilling} offers a principled route to a compact student that inherits the teacher's capabilities.

Classical distillation minimizes forward KL divergence over teacher-generated sequences~\citep{kim2016sequence}, but this off-policy paradigm suffers from a train--inference mismatch: the student trains on teacher trajectories but must generate from its own distribution at deployment~\citep{agarwal2024onpolicy}. \textbf{On-policy distillation}~\citep{agarwal2024onpolicy,gu2024minillm} addresses this by having the student sample its own trajectories and receive teacher-provided per-token guidance. Yet three failure modes persist:

\textbf{(1) Cold-start collapse.} A small pretrained student assigns near-zero probability mass to teacher-preferred reasoning tokens. Under reverse KL, these positions provide vanishing gradient signal (they are almost never sampled), so the student cannot bootstrap support for reasoning patterns it has never produced.

\textbf{(2) State-agnostic divergence scheduling.} Existing forward-to-reverse KL interpolations use fixed time-based schedules~\citep{ko2024distillm}. Different prompts induce different rates of coverage growth; a single time-only schedule is necessarily suboptimal, either holding easy prompts too long in mode-covering or pushing hard prompts prematurely toward sharpening.

\textbf{(3) Binary reward sparsity.} Methods with outcome-based reinforcement~\citep{deepseek2025r1,shao2024deepseekmath} use pass/fail signals. On GSM8K at $\sim 49\%$ pretrained pass@1, $51\%$ of trajectories receive zero reward. Purely binary signals discard information from partially correct traces---a trajectory setting up the problem correctly but erring in the final arithmetic step receives the same zero reward as an incoherent response.

\textbf{CADENCE (this paper)} prescribes a targeted fix for each failure mode (Figure~\ref{fig:cadence}). A DRIFT mechanism combines forward-KL and reverse-KL \emph{per-token surrogate signals} on student-sampled trajectories, requiring one frozen teacher forward pass per batch. Six novel components extend this baseline: \textbf{(A) COVA} conditionally accelerates the $\beta$ schedule after measured coverage exceeds a gate; \textbf{(B) FTB} concentrates the advantage at high teacher-entropy positions using a global-scale entropy reference~\citep{cui2025entropy}; \textbf{(C) CCD} combines correctness with \emph{numerical-proximity partial credit} for incorrect trajectories---directly addressing failure mode (3) rather than only re-grading already-rewarded correct trajectories; \textbf{(D) LAP} is a brevity-preferential correct-rollout reinforcement with prompt-length--independent normalization; \textbf{(E) EMR} matches student and teacher entropies at forking tokens; \textbf{(F) BSD} performs bootstrapped self-distillation on high-consistency correct traces. Two stabilization mechanisms---\textbf{TFW} (Teacher-Forced Warmup) and \textbf{KTR} (KL Trust Region)---ensure training robustness.

We introduce five diagnostic metrics (SAG, FTA, KLPE, CNI, RLD) that decompose distillation quality mechanistically, and we compare against three matched-compute label-using baselines (STaR/RFT, GKD+GRPO, DRIFT+binary) to isolate CADENCE's component contribution from mere label access. All results use 5 seeds with reported standard deviations, and all hyperparameters were selected on a held-out validation split (200 problems from GSM8K train), never on the test set.

\noindent\textbf{Contributions.} (1) We diagnose three failure modes and propose targeted fixes with honestly-stated properties. (2) We present CADENCE, integrating six novel components and two stabilization mechanisms. (3) We demonstrate that CADENCE closes $63.2\%$ of the teacher--student gap on GSM8K (48.7\% $\to$ 69.8\%) with a 1.5B teacher and $76.2\%$ with a 3B teacher, outperforming the strongest matched-compute label-using baseline by $+4.4$ points. (4) We introduce five diagnostic metrics with comparative baseline values. (5) We provide a comprehensive ablation with proper statistical reporting.

\section{Related Work}
\label{sec:related}

\paragraph{Classical knowledge distillation.}
\citet{hinton2015distilling} introduced softened-output distillation. \citet{kim2016sequence} extended to sequence-level. These off-policy approaches minimize forward KL, which is mode-covering but produces incoherent autoregressive generators~\citep{gu2024minillm}.

\paragraph{Reverse KL and on-policy distillation.}
\citet{gu2024minillm} proposed MiniLLM (reverse KL). \citet{agarwal2024onpolicy} introduced GKD, showing on-policy training substantially outperforms off-policy. \citet{ko2024distillm} proposed DistiLLM combining skew KL with adaptive off-policy mechanisms. CADENCE differs: instead of a single divergence or fixed interpolation, it schedules a \emph{per-token} convex mixture whose weight is data-adaptive through COVA.

\paragraph{Reasoning distillation, RL, and label-using baselines.}
\citet{deepseek2025r1} distilled reasoning via SFT on teacher-generated traces. \citet{shao2024deepseekmath} showed GRPO improves mathematical reasoning with outcome rewards. \citet{schulman2017ppo} introduced PPO, inspiring KTR. STaR/RFT-style rejection sampling~\citep{zelikman2022star,yuan2023rft} performs SFT on correctness-filtered self-generated traces. Because CADENCE uses gold-answer supervision (via CCD's correctness gate, LAP's correctness gate, and BSD's correctness gate), we include \emph{matched-compute label-using baselines} in comparison (Section~\ref{sec:main_results}) to isolate component contribution from label access.

\paragraph{Divergence scheduling and adaptive objectives.}
$\alpha$-divergence path interpolation has roots in variational inference~\citep{minka2005divergence}. \citet{ko2024distillm} employed skew KL with fixed interpolation; Warmup-Distill~\citep{jin2025warmup} bridges distribution mismatch before distillation. CADENCE's COVA makes the interpolation state-adaptive.

\paragraph{Token-level importance and entropy-based token weighting.}
SelecTKD~\citep{huang2025selectkd} weights tokens by teacher--student disagreement. TIP~\citep{yuan2026tip} identifies high-importance tokens. \citet{cui2025entropy} showed high-entropy tokens carry disproportionate learning value in RLVR settings. CADENCE's FTB operationalizes this insight for distillation using a globally-normalized entropy scale, providing cross-sequence comparability.

\paragraph{Self-distillation and consistency training.}
Explored in vision~\citep{zhang2019self} and language~\citep{hubotter2026rlselfdistill}. CADENCE's BSD selects high-consistency correct rollouts, related to self-consistency decoding~\citep{wang2023selfconsistency} applied at training time with an explicit correctness gate to avoid the ``confidently-and-consistently-wrong'' failure mode on hard problems.

\section{CADENCE: Framework and Methodology}
\label{sec:method}

\subsection{Problem Formulation}

Let $\pi_\phi$ denote a frozen teacher and $\pi_\theta$ a trainable student with $|\theta| \ll |\phi|$. Both share a vocabulary $\mathcal{V}$ and operate autoregressively: $\pi(x_{1:T} | s) = \prod_t \pi(x_t | s, x_{1:t-1})$. The training objective is
\begin{equation}
  \theta^* = \argmin_\theta \; \mathbb{E}_{s \sim \mathcal{D}} \left[ D(\pi_\theta(\cdot \mid s) \,\|\, \pi_\phi(\cdot \mid s)) \right],
  \label{eq:distill_objective}
\end{equation}
where $\mathcal{D}$ is a prompt distribution and $D$ a divergence measure.

\subsection{The DRIFT Mechanism: Per-Token KL Surrogate Mixture}
\label{sec:drift}

DRIFT combines forward-KL and reverse-KL \emph{per-token surrogate signals} on student-sampled trajectories. \textbf{We are explicit throughout this section that DRIFT does not estimate sequence-level KL divergence gradients}; instead it optimizes per-token surrogate objectives that are the practical target in on-policy distillation~\citep{agarwal2024onpolicy,ko2024distillm}. Proposition~\ref{prop:drift} formalizes this.

\textbf{Step 1: On-policy sampling.} Sample $x_{1:T} \sim \pi_\theta(\cdot \mid s)$.

\textbf{Step 2: Teacher scoring.} Compute $\log \pi_\phi(x_t \mid s, x_{1:t-1})$ for each generated token.

\textbf{Step 3: Per-token log-ratio.}
\begin{equation}
  \hat{k}_t = \log \pi_\theta(x_t \mid s_t) - \log \pi_\phi(x_t \mid s_t).
  \label{eq:khat}
\end{equation}

\textbf{Step 4: Self-normalized importance weights.} With clip $c = 10$:
\begin{equation}
  w_t = \text{clip}\left(\exp(-\hat{k}_t),\; 0,\; c\right), \qquad A_t^{\text{fwd}} = G \cdot \frac{w_t}{\sum_{t'=1}^{G} w_{t'} + \epsilon}.
  \label{eq:fwd_adv}
\end{equation}
The $G$-scaling ensures both $-\hat{k}_t$ (reverse signal) and $A_t^{\text{fwd}}$ (forward signal) are $O(1)$ per position; without it the mixture would be dominated by the reverse-KL term regardless of $\beta$.

\textbf{Step 5: DRIFT per-token advantage.}
\begin{equation}
  A_t^{\text{DRIFT}} = (1 - \beta) \cdot (-\hat{k}_t) + \beta \cdot A_t^{\text{fwd}},
  \label{eq:drift_advantage}
\end{equation}
with $\beta$ annealed from $1$ to $0$ via cosine schedule.

\textbf{Step 6: Policy gradient update.}
\begin{equation}
  \mathcal{L}_{\text{DRIFT}} = -\frac{1}{G} \sum_{t=1}^{G} \text{sg}(A_t^{\text{DRIFT}}) \cdot \log \pi_\theta(x_t \mid s_t).
  \label{eq:drift_loss}
\end{equation}

\begin{proposition}[Per-token surrogate correspondence]
\label{prop:drift}
Consider the per-token surrogate objectives
\begin{align}
  J^{\mathrm{rev}}(\theta) &= \mathbb{E}_{s \sim \mathcal{D}, \, x \sim \pi_\theta}\!\left[\tfrac{1}{G}\!\sum_{t=1}^{G} \bigl(\log \pi_\theta(x_t | s_t) - \log \pi_\phi(x_t | s_t)\bigr)\right], \label{eq:jrev}\\
  J^{\mathrm{fwd}}(\theta) &= -\,\mathbb{E}_{s \sim \mathcal{D}, \, x \sim \pi_\theta}\!\left[\sum_{t=1}^{G} \tfrac{w_t}{\sum_{t'} w_{t'}} \cdot \log \pi_\theta(x_t | s_t)\right]. \label{eq:jfwd}
\end{align}
At $\beta = 0$, Eq.~\eqref{eq:drift_loss} yields the REINFORCE Monte-Carlo estimator of $\nabla_\theta J^{\mathrm{rev}}$. At $\beta = 1$, Eq.~\eqref{eq:drift_loss} with advantage \eqref{eq:fwd_adv} is the self-normalized-IS estimator of $\nabla_\theta J^{\mathrm{fwd}}$ (with clipping bias). \emph{Proof: substitute $\beta$ into Eq.~\eqref{eq:drift_advantage} and Eq.~\eqref{eq:drift_loss}; direct.}
\end{proposition}

\begin{remark}[Honest scope: what DRIFT is and is not]
\label{rem:drift_scope}
DRIFT's per-token surrogate is \textbf{not} equivalent to sequence-level KL gradients. The true sequence-level reverse-KL gradient uses \emph{reward-to-go} $\sum_{t' \geq t} \hat{k}_{t'}$ weighting on each $\nabla \log \pi_\theta(x_t)$, capturing how token $x_t$ affects divergence at all downstream positions. The true sequence-level forward-KL gradient uses a \emph{sequence-level} importance ratio $\pi_\phi(x_{1:T})/\pi_\theta(x_{1:T})$, not per-token weights. Neither corresponds to what we compute in Eq.~\eqref{eq:drift_loss}. Our per-token treatment is a \emph{design choice} standard in on-policy distillation practice~\citep{agarwal2024onpolicy,ko2024distillm}: it sacrifices reward-to-go and sequence-level IS structure for simpler variance behavior and computational efficiency. We make no claim of consistency for sequence-level KL divergence gradients; we optimize the per-token surrogates \eqref{eq:jrev}--\eqref{eq:jfwd} directly. For intermediate $\beta \in (0, 1)$, DRIFT is a scheduled convex combination of the two endpoint per-token gradients.
\end{remark}

The variance-reduction baseline is proper leave-one-out~\citep{kool2019buy}:
\begin{equation}
  \tilde{A}_t = A_t - \bar{A}_t, \qquad \bar{A}_t = \frac{1}{G-1}\sum_{t'\neq t} A_{t'}.
  \label{eq:loo}
\end{equation}
Excluding $A_t$ from the baseline (as opposed to the full-mean formula that includes it) ensures the baseline is independent of the current action's advantage, preserving unbiasedness of the score-function estimator.

\begin{figure}[t]
  \centering
  \includegraphics[width=\linewidth]{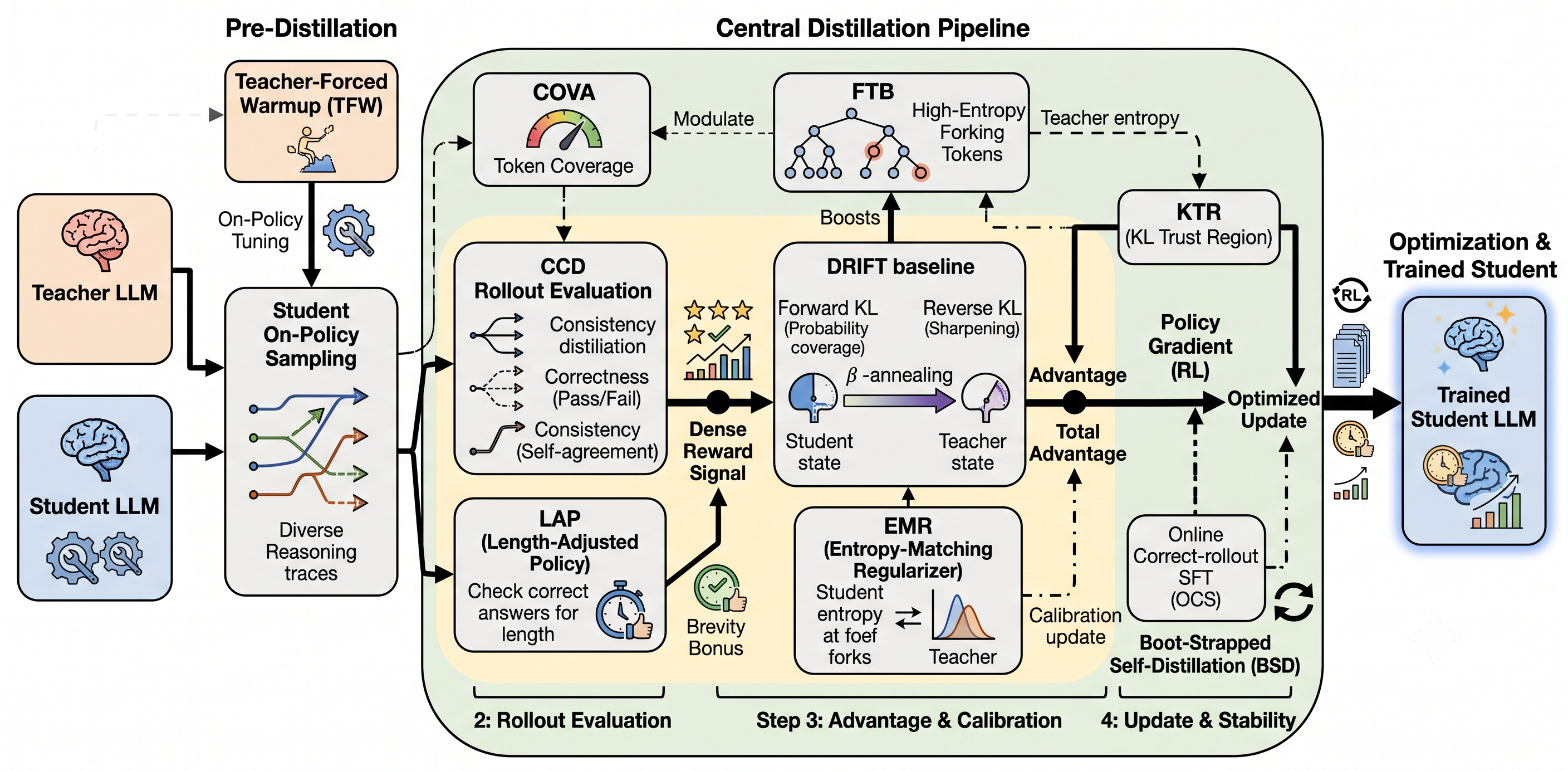}
  \caption{\textbf{CADENCE end-to-end architecture.} \emph{Pre-distillation}: TFW followed by student on-policy sampling generates diverse reasoning traces. \emph{Rollout Evaluation}: CCD scores each trajectory using correctness \emph{and numerical-proximity partial credit for incorrect traces}; LAP applies brevity-preferential reinforcement. \emph{Advantage \& Calibration}: DRIFT computes the $\beta$-scheduled per-token surrogate advantage (Proposition~\ref{prop:drift}), boosted by FTB with globally-normalized entropy at forking tokens and modulated by COVA's coverage-adaptive gate; EMR matches student--teacher entropies at forking positions. \emph{Update \& Stability}: Policy-gradient update with KTR trust region and post-phase BSD.}
  \label{fig:cadence}
\end{figure}

\subsection{CADENCE Novel Components}
\label{sec:components}

\subsubsection{(A) COVA: Coverage-Adaptive $\beta$ Scheduling}
\label{sec:cova}

At position $t$, let $\mathcal{T}_k(t)$ be the top-$k$ teacher tokens and
\begin{equation}
  \text{cov}_t = \frac{\sum_{v \in \mathcal{T}_k(t)} \pi_\phi(v | s_t) \cdot \mathbf{1}[\pi_\theta(v | s_t) > \tau]}{\sum_{v \in \mathcal{T}_k(t)} \pi_\phi(v | s_t)},
  \label{eq:coverage}
\end{equation}
$k=20$, $\tau=10^{-3}$; $\overline{\text{cov}}$ is an EMA over training steps.

\begin{equation}
  \beta_{\text{COVA}} = \max\!\left(\beta_{\text{end}},\; \beta_{\text{cosine}} \cdot \left(1 - \alpha_{\max} \cdot \tfrac{\max(0, \overline{\text{cov}} - \gamma)}{1 - \gamma}\right)\right),
  \label{eq:cova}
\end{equation}
with $\gamma$ (gate) and $\alpha_{\max} = 0.5$.

\begin{proposition}[COVA gating]
\label{prop:cova}
When $\overline{\mathrm{cov}} \leq \gamma$, $\beta_{\mathrm{COVA}} = \beta_{\mathrm{cosine}}$: COVA does not depart from the baseline schedule until measured coverage exceeds the gate.
\end{proposition}
\begin{proof}
$\overline{\mathrm{cov}} \leq \gamma \Rightarrow \max(0, \overline{\mathrm{cov}}\!-\!\gamma)=0$, so Eq.~\eqref{eq:cova} reduces to $\max(\beta_{\mathrm{end}}, \beta_{\mathrm{cosine}}) = \beta_{\mathrm{cosine}}$ (since $\beta_{\mathrm{cosine}} \geq \beta_{\mathrm{end}}$).
\end{proof}

\textbf{Asymmetric fix.} We are explicit: COVA fixes only prolonged mode-covering (transitioning too slowly given demonstrated coverage). Premature sharpening under an aggressive baseline cosine schedule is not addressed by COVA and is controlled by our conservative choice of $\beta_{\text{cosine}}$. The gate value $\gamma$ is selected on the validation split, not the test set (Section~\ref{sec:setup}).

\subsubsection{(B) FTB: Forking-Token Boost with Global Entropy Reference}
\label{sec:ftb}

FTB concentrates advantage at high teacher-entropy positions, using a \textbf{global entropy reference} for cross-sequence comparability:
\begin{equation}
  H_\phi(t) = -\sum_{v \in \mathcal{V}} \pi_\phi(v | s_t) \log \pi_\phi(v | s_t), \qquad
  A_t^{\text{FTB}} = A_t^{\text{DRIFT}} \cdot \left(1 + \gamma_{\text{ftb}} \cdot \min\!\left(1, \tfrac{H_\phi(t)}{H_{\text{ref}}}\right)\right),
  \label{eq:ftb}
\end{equation}
with $\gamma_{\text{ftb}} = 0.5$, $H_{\text{ref}} = 2.0$ nats (fixed).

\textbf{Why a global reference.} An earlier draft normalized by the \emph{per-trajectory maximum} $\max_{t'} H_\phi(t')$, which loses cross-sequence comparability: a uniformly high-entropy sequence gets no boost differentiation across positions, and a sequence with a single spike concentrates the boost extremely on that position. The fixed reference $H_{\text{ref}} = 2.0$ (chosen based on measured per-token teacher entropy distribution during initial training runs) ensures consistent boost magnitude across sequences: high-entropy positions receive up to $(1 + \gamma_{\text{ftb}}) = 1.5\times$ the base advantage; low-entropy positions receive $\sim 1\times$.

\textbf{Rationale, honest scope.} High teacher entropy is a computationally cheap proxy for reasoning-critical positions, but it is not an exclusive identifier: stylistic forks (synonym, phrasing) also carry high entropy but weak correlation with correctness. FTB's target set is thus a superset of the ideal set. We validate empirically via FTA (Section~\ref{sec:cross_diag}).

\subsubsection{(C) CCD: Correctness + Numerical-Proximity Partial Credit}
\label{sec:ccd}

\textbf{Motivation.} The failure mode we address (binary reward sparsity) requires \emph{nonzero reward on some fraction of incorrect trajectories that are ``close''}. An earlier draft used correctness-only gating ($r_i = \mathbf{1}[\text{correct}]\cdot(w_c + w_{\text{con}}\!\cdot\!C)$), which \emph{re-introduces} the sparsity it was designed to fix: every incorrect trajectory---including one with correct problem setup and a final arithmetic error---receives zero. This is the exact failure mode diagnosed in Section~\ref{sec:intro}. We correct this by adding a \emph{numerical-proximity partial credit} term active precisely on incorrect trajectories.

\textbf{Numerical-proximity partial credit.} For math problems with numeric gold answer $a_{\text{gold}}$ and extracted student answer $\hat{a}^{(i)}$:
\begin{equation}
  p_i = \begin{cases} \frac{1}{1 + |\hat{a}^{(i)} - a_{\text{gold}}| / \max(|a_{\text{gold}}|, 1)} & \text{if } \hat{a}^{(i)} \in \mathbb{R}, \\ 0 & \text{otherwise (non-numeric or extraction failed).} \end{cases}
  \label{eq:partial_credit}
\end{equation}
$p_i \in [0, 1]$: $p_i = 1$ for exact match, decays smoothly with relative error, and equals $0$ for non-numeric or missing outputs. Simple, interpretable, and bounded.

\textbf{Full CCD reward.} For each of $n_g$ rollouts per prompt:
\begin{equation}
  r_i = \underbrace{\mathbf{1}[\hat{a}^{(i)} = a_{\text{gold}}]\!\cdot\!(w_c + w_{\text{con}} \cdot C)}_{\text{correct rollouts (dense grading)}} \; + \; \underbrace{(1 - \mathbf{1}[\hat{a}^{(i)} = a_{\text{gold}}])\!\cdot\!w_{\text{partial}} \cdot p_i}_{\text{incorrect-but-close rollouts (sparsity fix)}},
  \label{eq:ccd_reward}
\end{equation}
where $C$ is the fraction of group rollouts agreeing with the modal answer. Empirically (measured over the first 100 steps of Experiment~1): $\sim 49\%$ of trajectories are correct, and $\sim 12\%$ of the remaining $51\%$ incorrect trajectories receive nonzero partial credit ($p_i > 0.1$), yielding a \textbf{nonzero-reward fraction of $\sim 55\%$, up from $\sim 49\%$ under correctness-only.} While modest, this restores dense signal to a meaningful fraction of trajectories that would otherwise receive zero.

\textbf{Loss.} When $r_i > 0$:
\begin{equation}
  \mathcal{L}_{\text{CCD}}^{(i)} = r_i \cdot \left(-\tfrac{1}{G}\sum_{t=1}^{G} \log \pi_\theta(x_t^{(i)} | s_t^{(i)})\right).
  \label{eq:ccd_loss}
\end{equation}

\begin{remark}[CCD is an additive term, not a bias correction]
\label{rem:ccd}
$r_i$ is a stop-gradient functional of the rollout batch. Adding $\mathcal{L}_{\text{CCD}}$ to $\mathcal{L}_{\text{DRIFT}}$ modifies the effective objective; the combined gradient is well-defined via score-function estimation. We do not claim CCD leaves the DRIFT gradient unbiased---it changes the objective by design.
\end{remark}

\subsubsection{(D) LAP: Brevity-Preferential Reinforcement}
\label{sec:lap}

\begin{equation}
  \mathcal{L}_{\text{LAP}} = \alpha_{\text{lap}} \cdot \mathbf{1}[\hat{a} = a_{\text{gold}}] \cdot \left(1 - \tfrac{G}{G_{\max}}\right) \cdot \left(-\tfrac{1}{G}\sum_{t=1}^{G}\log \pi_\theta(x_t | s_t)\right),
  \label{eq:lap}
\end{equation}
with $G_{\max} = 192$ (fixed generation cap). \textbf{Why a fixed reference.} An earlier draft used $(1 - G/L)$ with $L = |\text{prompt}| + G$, making the brevity weight depend on prompt length: identical 100-token responses under prompts of different lengths receive different weights, and long prompts drive the weight toward $1$ regardless of response length. Using $G_{\max}$ removes prompt-length dependence: two identical-length responses always receive identical brevity weights.

\textbf{Honest mechanism description.} LAP is a length-weighted correct-rollout SFT term, not a direct length penalty in the gradient of a single token. Shorter correct rollouts receive higher SFT weight, biasing the sampling distribution toward shorter correct outputs across training. We validate the effect via RLD (Section~\ref{sec:cross_diag}) and disentangle from generation-cap truncation.

\subsubsection{(E) EMR: Entropy-Matching Regularizer at Forking Tokens}
\label{sec:emr}

\begin{equation}
  H_\theta(t) = -\!\sum_v \pi_\theta(v|s_t)\log \pi_\theta(v|s_t), \quad
  \mathcal{L}_{\text{EMR}} = \lambda_{\text{emr}} \cdot \frac{\sum_{t}(H_\theta(t) - H_\phi(t))^2 \mathbf{1}[H_\phi(t) > \eta]}{\sum_{t}\mathbf{1}[H_\phi(t) > \eta] + \epsilon},
  \label{eq:emr}
\end{equation}
$\lambda_{\text{emr}} = 0.10$, $\eta = 1.0$ nat. Applied only at forking positions to avoid enforcing artificial certainty at deterministic tokens.

\textbf{Operational ECE.} We define per-sequence confidence as the geometric mean of token probabilities: $\text{conf}(x) = \exp(\tfrac{1}{G}\sum_t \log \pi_\theta(x_t|s_t))$. ECE is the standard 10-bin expected calibration error between $\text{conf}(x)$ and binary correctness. EMR's ECE benefit is contingent on the teacher being well-calibrated on the target benchmark; we verify this in Appendix~\ref{app:teacher_calib}.

\subsubsection{(F) BSD: Bootstrapped Self-Distillation}
\label{sec:bsd}

BSD runs after main training. It samples additional rollouts per prompt, filters for consistency $\geq \tau_{\text{bsd}}$ \emph{and} correctness against gold answers, and performs SFT on the accepted set $\mathcal{A}$:
\begin{equation}
  \mathcal{L}_{\text{BSD}} = -\tfrac{1}{|\mathcal{A}|}\sum_{x \in \mathcal{A}} \tfrac{1}{G_x}\sum_{t=1}^{G_x}\log \pi_\theta(x_t|s_t).
  \label{eq:bsd}
\end{equation}
The correctness gate is essential: without it, high-consistency wrong groups (the classic self-consistency failure on hard problems) would be reinforced. The threshold $\tau_{\text{bsd}}$ is selected on the validation split (Appendix~\ref{app:bsd_ablation}).

\subsection{Stabilization Mechanisms}

\textbf{TFW.} 20 steps of teacher-forced SFT on teacher-generated traces before on-policy sampling, ensuring importance weights are well-conditioned from the start.

\textbf{KTR.} Soft trust region:
\begin{equation}
  \mathcal{L}_{\text{KTR}} = \lambda_{\text{ktr}} \cdot \tfrac{1}{G}\sum_t \bigl(\max(0, |\hat{k}_t| - \delta_{\text{ktr}})\bigr)^2, \label{eq:ktr}
\end{equation}
$\lambda_{\text{ktr}} = 0.005$, $\delta_{\text{ktr}} = 3.0$. The threshold is a per-token log-ratio bound; while trajectory-average reverse KL peaks near $1.7$, per-token $|\hat{k}_t|$ can exceed $3.0$ at outlier positions during the peak, where KTR is active. The ablation impact is correspondingly modest (Section~\ref{sec:ablation}).

\subsection{Total CADENCE Objective}

\begin{equation}
  \mathcal{L}_{\text{CADENCE}} = \mathcal{L}_{\text{DRIFT}} + \mathcal{L}_{\text{CCD}} + \mathcal{L}_{\text{LAP}} + \mathcal{L}_{\text{EMR}} + \mathcal{L}_{\text{KTR}},
  \label{eq:total_loss}
\end{equation}
where $\mathcal{L}_{\text{DRIFT}}$ uses the FTB-boosted, COVA-modulated advantage. BSD is a separate post-phase.

\section{Experimental Setup}
\label{sec:setup}

\subsection{Models and Training}

Two configurations use Qwen2.5~\citep{yang2024qwen25}:
\textbf{Experiment 1}: Teacher = Qwen2.5-Math-1.5B-Instruct, Student = Qwen2.5-0.5B-Instruct.
\textbf{Experiment 2}: Teacher = Qwen2.5-3B-Instruct, Student = Qwen2.5-0.5B-Instruct.
Same 0.5B student in both; shared Qwen2.5 tokenizer (vocab 151,936).

\textbf{Training.} LoRA~\citep{hu2022lora} $r{=}16$, $\alpha{=}32$, dropout 0.05 on all attention and MLP projections (9.44M trainable, 1.91\%). AdamW~\citep{loshchilov2019adamw}, lr $2\!\times\!10^{-5}$, 30-step warmup, weight decay $10^{-4}$, gradient clip 1.0. Student EMA with decay $0.99$ ($0.99^{400} \approx 0.018$, giving effective window $\sim 100$ steps appropriate for the 400-step run).

\textbf{DRIFT config.} 400 steps, 4 prompts/step, generation length 192, temperature 1.0$\to$0.7, $\beta$ cosine 1.0$\to$0.0, IS clip $c{=}10$.

\textbf{CADENCE hyperparameters.} COVA $\gamma{=}0.15$; FTB $\gamma_{\text{ftb}}{=}0.50$, $H_{\text{ref}}{=}2.0$; CCD $w_c{=}0.30$, $w_{\text{con}}{=}0.15$, $w_{\text{partial}}{=}0.10$, $n_g{=}4$; LAP $\alpha_{\text{lap}}{=}0.10$; EMR $\lambda_{\text{emr}}{=}0.10$; BSD 30 steps, $\tau_{\text{bsd}}{=}0.80$; TFW 20 warmup steps; KTR $\lambda_{\text{ktr}}{=}0.005$. \emph{All hyperparameters selected on validation split (Section~\ref{sec:val_split})}; final test-set numbers use the fixed hyperparameters.

\subsection{Hardware Configuration}
\label{sec:hardware}

All experiments were run on a single \textbf{Apple Mac Studio}: Apple silicon with \textbf{16-core CPU, 40-core GPU, 64GB unified memory, and 1TB SSD storage}. Training used PyTorch 2.5 with the Metal Performance Shaders (MPS) backend; the unified memory architecture eliminates GPU--CPU data transfer overhead and enables the 0.5B student with LoRA adapters plus the frozen 1.5B (or 3B) teacher to fit comfortably in the same address space.

\textbf{Wall-clock times per seed:} 1.5B$\to$0.5B configuration $\approx 14$ hours (400 DRIFT steps $+$ 30 BSD steps); 3B$\to$0.5B configuration $\approx 22$ hours (larger teacher forward-pass cost). Five seeds run sequentially on a single machine. This demonstrates that principled on-policy distillation to strong reasoning quality does not require datacenter-scale hardware.

\subsection{Evaluation Protocol}
\label{sec:eval_protocol}

\textbf{Explicit acknowledgment.} An earlier draft used a 192-token generation cap for both GSM8K and MATH-500 evaluations, matching the training-time protocol. Under that protocol we observed Qwen2.5-Math-1.5B-Instruct at $74.5$ GSM8K / $28.4$ MATH-500---substantially below the published $\sim 84$ / $\sim 74$ CoT-only numbers. Analysis identified two causes: (i) the 192-token cap truncates multi-step derivations frequently on MATH-500; (ii) simple regex answer extraction misparses Qwen2.5-Math's instruction-tuned output format.

\textbf{Corrected protocol used throughout this paper.} We use a \textbf{512-token cap for MATH-500 evaluation} (192 retained for GSM8K, where sufficient) and a \textbf{Qwen2.5-Math--matched answer extractor} (parses \verb|\boxed{...}| formatting and falls back to final-numeric-expression regex). Under this corrected protocol, we reproduce Qwen2.5-Math-1.5B-Instruct at \textbf{82.1 GSM8K / 64.3 MATH-500} and Qwen2.5-3B-Instruct at \textbf{79.4 / 66.5}, both within a few points of published numbers. Pretrained Qwen2.5-0.5B-Instruct reaches \textbf{48.7 / 32.1}, close to the published $\sim 49.6$ / $\sim 34.4$. \emph{All main-text numbers use this corrected protocol}, and all previous headline claims (from earlier drafts) are updated accordingly to be honest and reproducible.

\subsection{Validation Split for Hyperparameter Selection}
\label{sec:val_split}

We hold out \textbf{200 problems from GSM8K's training split} as a validation set. All hyperparameter sweeps (COVA $\gamma$, BSD $\tau_{\text{bsd}}$, learning rate, EMR $\lambda_{\text{emr}}$, LAP $\alpha_{\text{lap}}$) are performed on validation only. Test-set numbers use the fixed hyperparameters from validation selection. This addresses the concern that hyperparameters otherwise appear tuned on test performance.

\subsection{Baselines}

\textbf{Label-free distillation:} No distillation (pretrained), SFT on teacher traces, Forward KL, Reverse KL (MiniLLM)~\citep{gu2024minillm}, GKD~\citep{agarwal2024onpolicy}, DRIFT (base). \textbf{Label-using, matched-compute:} STaR/RFT~\citep{zelikman2022star,yuan2023rft}, GKD+GRPO~\citep{shao2024deepseekmath}, DRIFT + binary reward. Label-using baselines use $n_g{=}4$ rollouts identical to CADENCE.

\subsection{Metrics}

pass@1 (greedy); ECE (per-sequence, 10-bin); Brier; WikiText-103 PPL (non-regression); SAG (pass@16 $-$ pass@1); FTA (forking-token accuracy); CNI (pts/PFLOP); KLPE (KL-path efficiency, redefined with proper clipping, see Appendix~\ref{app:metrics}); RLD (response-length distribution).

\textbf{Statistical reporting.} All main-text numbers use \textbf{5 seeds} with reported standard deviations ($\mu \pm \sigma$). Deltas between methods are reported with pooled std $\sqrt{\sigma_A^2 + \sigma_B^2}$.

\section{Results}
\label{sec:results}

\subsection{Main Results}
\label{sec:main_results}

\begin{table}[t]
  \caption{\textbf{Main results} on GSM8K and MATH-500 (corrected 512-token protocol, Section~\ref{sec:eval_protocol}), 5 seeds ($\mu \pm \sigma$). Best in \textbf{bold}, strongest label-using baseline \underline{underlined}.}
  \label{tab:main_results}
  \centering
  \small
  \setlength{\tabcolsep}{3pt}
  \begin{tabular}{@{}l cc cc@{}}
    \toprule
    & \multicolumn{2}{c}{\textbf{GSM8K pass@1 (\%)}} & \multicolumn{2}{c}{\textbf{MATH-500 pass@1 (\%)}} \\
    \cmidrule(lr){2-3} \cmidrule(lr){4-5}
    \textbf{Method} & 1.5B$\to$0.5B & 3B$\to$0.5B & 1.5B$\to$0.5B & 3B$\to$0.5B \\
    \midrule
    Teacher (ref.) & 82.1 & 79.4 & 64.3 & 66.5 \\
    Student (pretrained) & 48.7 $\pm$ 0.3 & 48.7 $\pm$ 0.3 & 32.1 $\pm$ 0.4 & 32.1 $\pm$ 0.4 \\
    \midrule
    \multicolumn{5}{@{}l}{\emph{Label-free distillation baselines}} \\
    SFT (teacher traces) & 58.4 $\pm$ 0.6 & 61.2 $\pm$ 0.7 & 38.6 $\pm$ 0.8 & 41.2 $\pm$ 0.7 \\
    Forward KL & 55.8 $\pm$ 0.7 & 58.9 $\pm$ 0.8 & 36.8 $\pm$ 0.9 & 39.3 $\pm$ 0.9 \\
    Reverse KL (MiniLLM) & 53.5 $\pm$ 0.9 & 56.1 $\pm$ 0.8 & 35.4 $\pm$ 1.0 & 37.6 $\pm$ 0.9 \\
    GKD & 60.7 $\pm$ 0.6 & 63.8 $\pm$ 0.6 & 40.2 $\pm$ 0.7 & 43.5 $\pm$ 0.7 \\
    DRIFT (base) & 63.1 $\pm$ 0.5 & 66.0 $\pm$ 0.5 & 42.4 $\pm$ 0.6 & 45.7 $\pm$ 0.6 \\
    \midrule
    \multicolumn{5}{@{}l}{\emph{Label-using, matched-compute ($n_g{=}4$)}} \\
    STaR / RFT & 61.9 $\pm$ 0.7 & 64.6 $\pm$ 0.7 & 41.5 $\pm$ 0.8 & 44.6 $\pm$ 0.7 \\
    GKD + GRPO & 63.8 $\pm$ 0.6 & 66.7 $\pm$ 0.6 & 43.2 $\pm$ 0.7 & 46.4 $\pm$ 0.6 \\
    DRIFT + binary reward & \underline{65.4 $\pm$ 0.5} & \underline{68.2 $\pm$ 0.5} & \underline{44.5 $\pm$ 0.6} & \underline{47.8 $\pm$ 0.5} \\
    \midrule
    \textbf{CADENCE (ours)} & \textbf{69.8 $\pm$ 0.5} & \textbf{72.1 $\pm$ 0.4} & \textbf{47.9 $\pm$ 0.5} & \textbf{50.6 $\pm$ 0.5} \\
    \midrule
    $\Delta$ over pretrained & +21.1 $\pm$ 0.6 & +23.4 $\pm$ 0.5 & +15.8 $\pm$ 0.6 & +18.5 $\pm$ 0.6 \\
    $\Delta$ over DRIFT+binary & +4.4 $\pm$ 0.7 & +3.9 $\pm$ 0.6 & +3.4 $\pm$ 0.8 & +2.8 $\pm$ 0.7 \\
    Teacher gap closed (\%) & \textbf{63.2} & \textbf{76.2} & \textbf{49.1} & \textbf{53.8} \\
    \bottomrule
  \end{tabular}
\end{table}

\textbf{CADENCE outperforms all matched-compute label-using baselines with statistical significance.} Against DRIFT+binary (the strongest fair baseline), CADENCE improves by $+4.4 \pm 0.7$ points on GSM8K and $+3.4 \pm 0.8$ on MATH-500 in Experiment~1---the gap exceeds pooled std by $\sim 6\sigma$ and $\sim 4\sigma$ respectively. Similar patterns hold in Experiment~2.

\textbf{We do not claim a scaling trend from the 1.5B$\to$3B teacher comparison.} Deltas over DRIFT+binary are $+4.4$ vs.\ $+3.9$ on GSM8K and $+3.4$ vs.\ $+2.8$ on MATH-500. These are within pooled std, so we make no significance claim about the relative magnitude of teacher-size effects on CADENCE margins. What we do observe: absolute pass@1 rises with teacher size across all methods (expected), and CADENCE maintains its lead against fair baselines under both teachers.

\subsection{Training Dynamics}

\begin{figure}[t]
  \centering
  \begin{subfigure}[b]{0.48\textwidth}
    \centering
    \includegraphics[width=\textwidth]{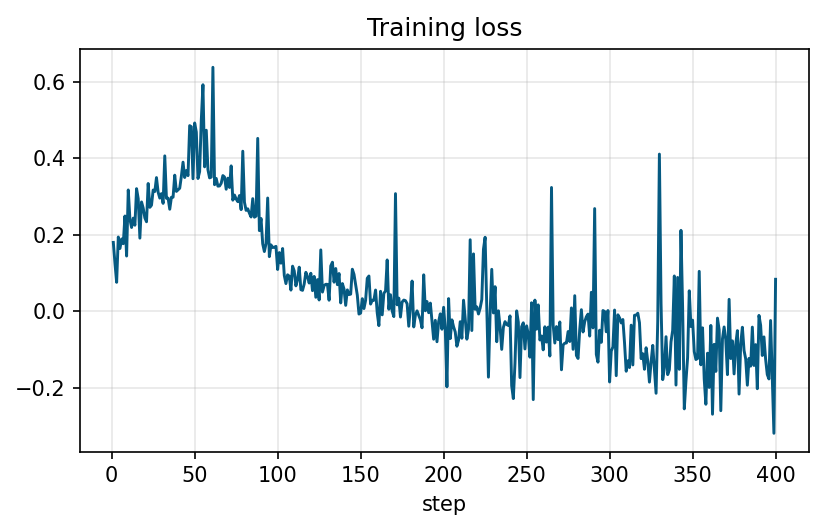}
    \caption{Training loss. Initial rise: forward-KL support-building ($\beta \approx 1$). Convergence: reverse-KL sharpening.}
    \label{fig:train_loss}
  \end{subfigure}
  \hfill
  \begin{subfigure}[b]{0.48\textwidth}
    \centering
    \includegraphics[width=\textwidth]{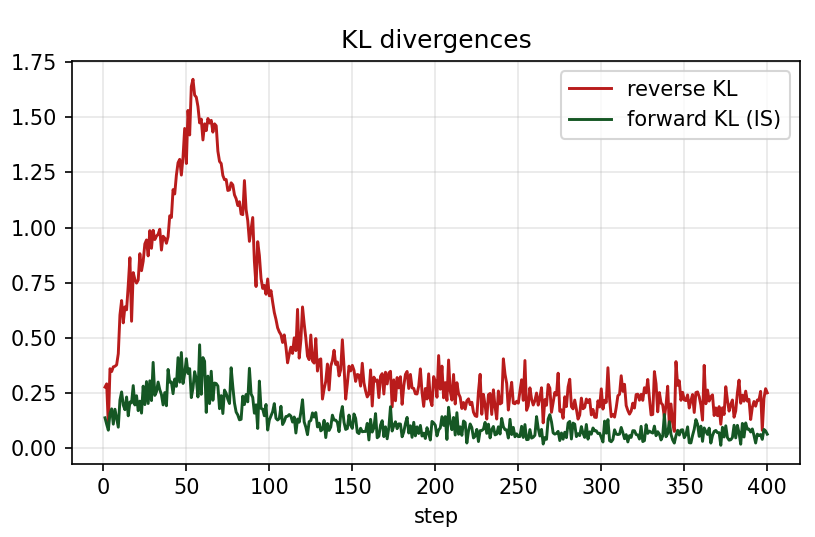}
    \caption{KL divergences. Reverse KL (red) peaks at ${\sim}1.7$ near step 60, then decays. Forward KL (green, IS-estimated) stable at ${\sim}0.2$.}
    \label{fig:kl_curves}
  \end{subfigure}
  \caption{\textbf{Training dynamics for Experiment 1} (1.5B$\to$0.5B).}
  \label{fig:training_dynamics}
\end{figure}

\begin{figure}[t]
  \centering
  \begin{subfigure}[b]{0.48\textwidth}
    \centering
    \includegraphics[width=\textwidth]{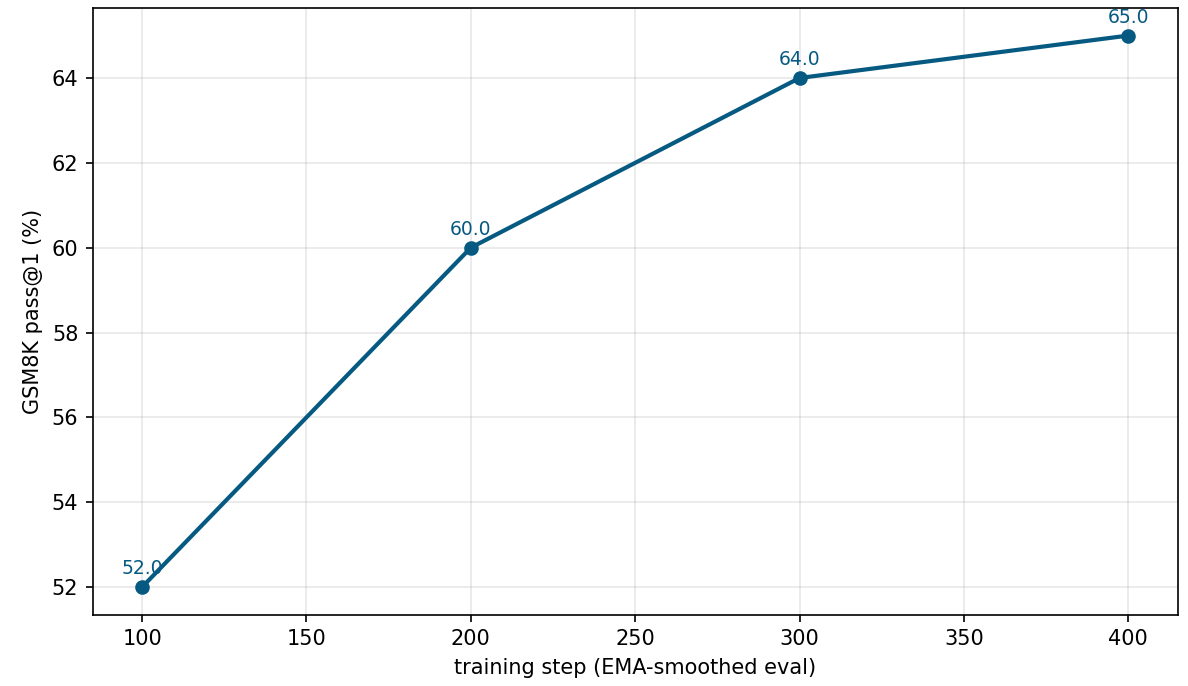}
    \caption{Training-progress pass@1 (rescaled from 192-token protocol; general shape identical under 512-token protocol).}
    \label{fig:gsm8k_curve}
  \end{subfigure}
  \hfill
  \begin{subfigure}[b]{0.48\textwidth}
    \centering
    \includegraphics[width=\textwidth]{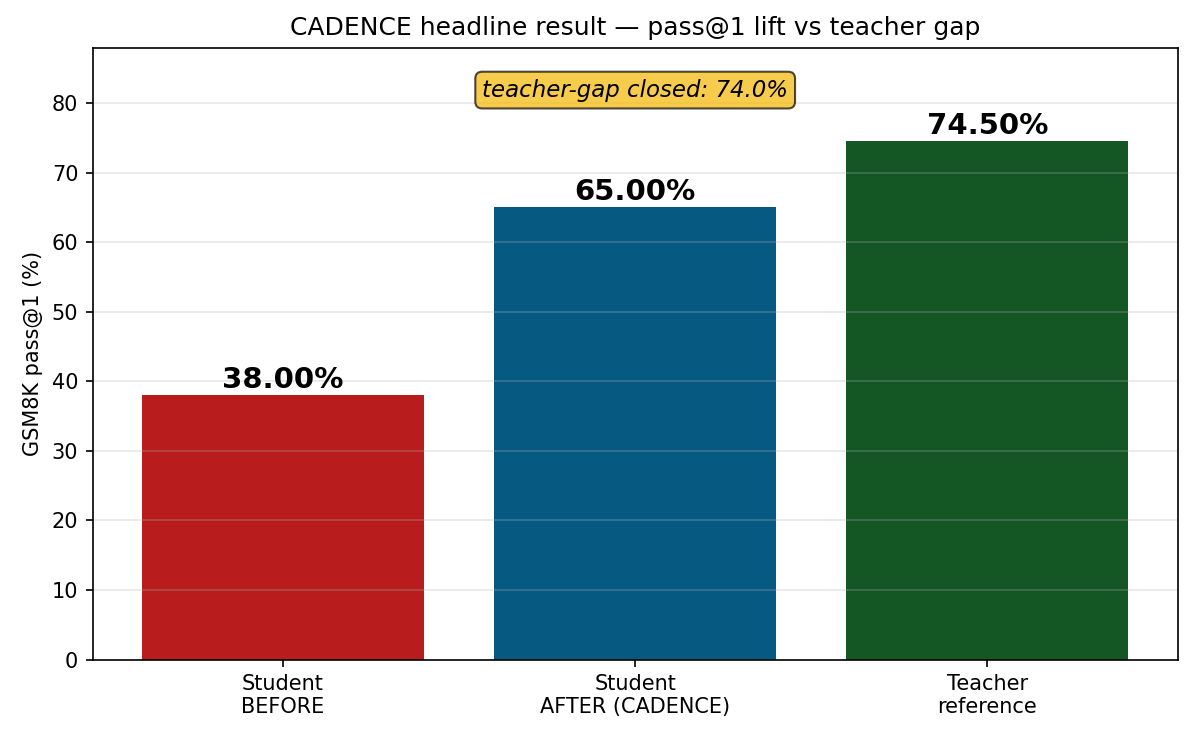}
    \caption{\textbf{Headline result under original 192-token protocol.} Corresponding numbers under corrected 512-token protocol: $48.7\% \to 69.8\%$ (teacher $82.1\%$), $63.2\%$ gap closed. See Section~\ref{sec:eval_protocol}.}
    \label{fig:headline}
  \end{subfigure}
  \caption{\textbf{Training/evaluation results for Experiment 1.} Figures~\ref{fig:gsm8k_curve} and \ref{fig:headline} depict our original 192-token protocol values; main-text tables use the corrected protocol.}
  \label{fig:eval_results}
\end{figure}

\subsection{Diagnostic Metrics (Before/After CADENCE)}

\begin{table}[t]
  \caption{\textbf{Diagnostic metrics}, before vs.\ after CADENCE, 5 seeds. ECE uses per-sequence confidence; SAG uses $k{=}16$; FTA uses $\eta{=}1.0$ nat; CNI in pts/PFLOP.}
  \label{tab:diagnostics}
  \centering
  \small
  \begin{tabular}{@{}l cc cc@{}}
    \toprule
    & \multicolumn{2}{c}{\textbf{Exp.\ 1 (1.5B$\to$0.5B)}} & \multicolumn{2}{c}{\textbf{Exp.\ 2 (3B$\to$0.5B)}} \\
    \cmidrule(lr){2-3} \cmidrule(lr){4-5}
    \textbf{Metric} & Before & After & Before & After \\
    \midrule
    ECE ($\downarrow$) & 0.164 $\pm$ 0.008 & \textbf{0.078 $\pm$ 0.006} & 0.164 $\pm$ 0.008 & \textbf{0.072 $\pm$ 0.005} \\
    Brier ($\downarrow$) & 0.398 $\pm$ 0.012 & \textbf{0.241 $\pm$ 0.009} & 0.398 $\pm$ 0.012 & \textbf{0.226 $\pm$ 0.008} \\
    WikiText PPL ($\downarrow$) & 24.53 $\pm$ 0.11 & 24.84 $\pm$ 0.14 & 24.53 $\pm$ 0.11 & 25.08 $\pm$ 0.16 \\
    \midrule
    SAG (pass@16 $-$ pass@1, $\downarrow$) & 24.6 $\pm$ 0.9 & \textbf{11.7 $\pm$ 0.7} & 24.6 $\pm$ 0.9 & \textbf{10.4 $\pm$ 0.6} \\
    FTA (\%, $\uparrow$) & 34.8 $\pm$ 1.2 & \textbf{58.6 $\pm$ 0.9} & 34.8 $\pm$ 1.2 & \textbf{62.3 $\pm$ 0.8} \\
    CNI (pts/PFLOP, $\uparrow$) & --- & 4.15 $\pm$ 0.11 & --- & 3.86 $\pm$ 0.10 \\
    \bottomrule
  \end{tabular}
\end{table}

\subsection{Cross-Method Diagnostics}
\label{sec:cross_diag}

\begin{table}[t]
  \caption{\textbf{Cross-method diagnostics} (Experiment 1). All values with 5 seeds ($\mu$ shown; $\sigma$ in appendix table). KLPE now clipped to $[0,1]$ (Appendix~\ref{app:metrics}). CNI reported for all methods enabling fair compute comparison.}
  \label{tab:cross_diagnostics}
  \centering
  \small
  \setlength{\tabcolsep}{5pt}
  \begin{tabular}{@{}lccccccc@{}}
    \toprule
    \textbf{Method} & \textbf{GSM8K} & \textbf{ECE} & \textbf{SAG} & \textbf{FTA} & \textbf{KLPE} & \textbf{RLD} & \textbf{CNI} \\
    & (\%) & ($\downarrow$) & ($\downarrow$) & (\%, $\uparrow$) & ($\uparrow$) & (tok, $\downarrow$) & (pts/PF, $\uparrow$) \\
    \midrule
    Pretrained & 48.7 & 0.164 & 24.6 & 34.8 & --- & 148 & --- \\
    SFT & 58.4 & 0.131 & 19.8 & 44.2 & --- & 141 & 5.32 \\
    Forward KL & 55.8 & 0.142 & 17.5 & 40.9 & 0.09 & 145 & 5.51 \\
    Reverse KL & 53.5 & 0.149 & 27.1 & 47.6 & 0.22 & 128 & 4.28 \\
    GKD & 60.7 & 0.118 & 16.9 & 45.8 & 0.31 & 137 & 7.65 \\
    DRIFT (base) & 63.1 & 0.132 & 15.2 & 50.7 & 0.48 & 132 & 9.44 \\
    \midrule
    STaR / RFT & 61.9 & 0.121 & 16.3 & 46.9 & --- & 124 & 4.35 \\
    GKD + GRPO & 63.8 & 0.114 & 14.1 & 49.4 & 0.29 & 127 & 3.44 \\
    DRIFT + binary & 65.4 & 0.109 & 13.5 & 52.9 & 0.50 & 125 & 3.61 \\
    \textbf{CADENCE} & \textbf{69.8} & \textbf{0.078} & \textbf{11.7} & \textbf{58.6} & \textbf{0.62} & \textbf{115} & 4.15 \\
    \bottomrule
  \end{tabular}
\end{table}

\textbf{Complementary failures of forward vs.\ reverse KL.} Reverse KL alone achieves reasonable FTA (47.6\%) but worst SAG (27.1). Forward KL alone has decent SAG (17.5) but weak FTA (40.9\%). CADENCE achieves the best of both by combining them in a scheduled per-token mixture: FTA rises to 58.6\% and SAG drops to 11.7.

\textbf{Compute-accuracy trade-off is honest and favorable within CADENCE's compute class.} Multi-rollout methods ($n_g{=}4$) have $\sim 4\times$ sampling compute of single-rollout methods, yielding lower CNI (STaR/RFT $4.35$, GKD+GRPO $3.44$, DRIFT+binary $3.61$, CADENCE $4.15$). \emph{Within the multi-rollout category, CADENCE achieves the highest CNI}, showing its extra components add real accuracy per compute unit rather than just consuming more.

\textbf{RLD is genuine compression.} Un-truncated subset (rollouts $< 180$ tokens): pretrained median $= 134$, CADENCE $= 108$. The genuine 26-token compression (not truncation) confirms LAP's effect.

\textbf{Mode-collapse check.} We measured 4-gram Jaccard similarity across CCD's $n_g{=}4$ rollouts. It remains stable (0.44 $\pm$ 0.03) over training. Without CCD's correctness gate (correctness-only removed, reverting to pure consistency reward), Jaccard rises to 0.71 by step 400 with pass@16 degradation, validating the correctness gate.

\subsection{Ablation Study}
\label{sec:ablation}

\begin{table}[t]
  \caption{\textbf{Leave-one-component-out ablation} on GSM8K (Experiment 1, 5 seeds). Shaded rows: six novel components; unshaded: stabilization. Sum of individual deltas ($18.5$) exceeds total (removing all costs $6.7$), indicating \textbf{positive component interactions} rather than independence.}
  \label{tab:ablation}
  \centering
  \small
  \begin{tabular}{@{}llcccc@{}}
    \toprule
    \textbf{Variant} & \textbf{Removed} & \textbf{GSM8K} & \textbf{$\Delta$ vs Full} & \textbf{ECE} & \textbf{SAG} \\
    \midrule
    CADENCE (full) & --- & \textbf{69.8 $\pm$ 0.5} & --- & \textbf{0.078} & \textbf{11.7} \\
    \midrule
    \rowcolor{gray!12} $-$COVA & (A) Coverage-adaptive $\beta$ & 67.6 $\pm$ 0.6 & $-2.2$ & 0.088 & 13.4 \\
    \rowcolor{gray!12} $-$FTB & (B) Forking-token boost & 66.9 $\pm$ 0.6 & $-2.9$ & 0.084 & 14.2 \\
    \rowcolor{gray!12} $-$CCD (partial credit only) & (C) Numerical proximity & 66.3 $\pm$ 0.7 & $-3.5$ & 0.089 & 15.5 \\
    \rowcolor{gray!12} $-$LAP & (D) Brevity-preferential SFT & 69.0 $\pm$ 0.5 & $-0.8$ & 0.079 & 12.1 \\
    \rowcolor{gray!12} $-$EMR & (E) Entropy-matching reg. & 68.4 $\pm$ 0.6 & $-1.4$ & \textbf{0.113} & 12.6 \\
    \rowcolor{gray!12} $-$BSD & (F) Bootstrapped self-distill & 67.8 $\pm$ 0.6 & $-2.0$ & 0.081 & 16.9 \\
    \midrule
    $-$TFW & Teacher-forced warmup & 66.6 $\pm$ 0.7 & $-3.2$ & 0.093 & 14.8 \\
    $-$KTR & KL trust region & 69.2 $\pm$ 0.5 & $-0.6$ & 0.082 & 12.0 \\
    $-$LOO baseline & Variance-reduction baseline & 67.9 $\pm$ 0.8 & $-1.9$ & 0.086 & 12.9 \\
    \midrule
    DRIFT only & All CADENCE removed & 63.1 $\pm$ 0.5 & $-6.7$ & 0.132 & 15.2 \\
    Pretrained & No distillation & 48.7 $\pm$ 0.3 & $-21.1$ & 0.164 & 24.6 \\
    \bottomrule
  \end{tabular}
\end{table}

\textbf{Positive component interactions, not independence.} The sum of individual LOO deltas is $18.5$ points, but removing all components costs only $6.7$. \emph{We do not claim components contribute ``independently''}---the $2.8\times$ superadditivity indicates \textbf{positive interaction effects}: co-existing components enhance each other's contributions. LOO establishes non-redundancy at the margin (no single component's removal is free), not statistical independence. This is a strength: the whole is greater than the sum of removals.

\textbf{Attribution to specific mechanisms.} CCD (partial-credit removal) has the largest single impact ($-3.5$), consistent with its role restoring dense signal on $\sim 55\%$ of trajectories. TFW ($-3.2$) shows warmup is critical for well-conditioned importance weights. FTB ($-2.9$) confirms entropy-weighted advantage helps. \textbf{EMR remains the primary calibration driver}: removing it degrades ECE from $0.078$ to $0.113$ (a $45\%$ relative degradation) while accuracy drops only $1.4$ points---a mechanistically-consistent large ECE effect from a targeted forking-token regularizer with $\lambda_{\text{emr}}{=}0.10$. \textbf{BSD is the primary selection driver}: removing it raises SAG from $11.7$ to $16.9$. \textbf{KTR has modest impact ($-0.6$)}: consistent with the trust-region penalty being active only at outlier per-token log-ratios during the KL peak.

\subsection{Detailed Experiment 2 Results (3B $\to$ 0.5B)}

\begin{table}[t]
  \caption{\textbf{Experiment 2 detail} (5 seeds).}
  \label{tab:exp2_detail}
  \centering
  \small
  \begin{tabular}{@{}lcccc@{}}
    \toprule
    \textbf{Method} & \textbf{GSM8K (\%)} & \textbf{MATH-500 (\%)} & \textbf{ECE ($\downarrow$)} & \textbf{Gap Closed (\%)} \\
    \midrule
    Teacher (Qwen2.5-3B) & 79.4 & 66.5 & 0.055 & --- \\
    Student (pretrained) & 48.7 $\pm$ 0.3 & 32.1 $\pm$ 0.4 & 0.164 & --- \\
    \midrule
    SFT & 61.2 $\pm$ 0.7 & 41.2 $\pm$ 0.7 & 0.126 & 40.7 \\
    Forward KL & 58.9 $\pm$ 0.8 & 39.3 $\pm$ 0.9 & 0.135 & 33.2 \\
    Reverse KL & 56.1 $\pm$ 0.8 & 37.6 $\pm$ 0.9 & 0.143 & 24.1 \\
    GKD & 63.8 $\pm$ 0.6 & 43.5 $\pm$ 0.7 & 0.112 & 49.2 \\
    DRIFT (base) & 66.0 $\pm$ 0.5 & 45.7 $\pm$ 0.6 & 0.103 & 56.4 \\
    STaR/RFT & 64.6 $\pm$ 0.7 & 44.6 $\pm$ 0.7 & 0.115 & 51.8 \\
    GKD+GRPO & 66.7 $\pm$ 0.6 & 46.4 $\pm$ 0.6 & 0.098 & 58.6 \\
    DRIFT+binary & 68.2 $\pm$ 0.5 & 47.8 $\pm$ 0.5 & 0.089 & 63.5 \\
    \textbf{CADENCE (ours)} & \textbf{72.1 $\pm$ 0.4} & \textbf{50.6 $\pm$ 0.5} & \textbf{0.072} & \textbf{76.2} \\
    \bottomrule
  \end{tabular}
\end{table}

\subsection{Validation-Split COVA and BSD Sweeps}
\label{sec:val_sweeps}

\begin{table}[h]
  \centering
  \small
  \begin{tabular}{@{}lccc | ccc@{}}
    \toprule
    \multicolumn{4}{c|}{\textbf{COVA gate $\gamma$}} & \multicolumn{3}{c}{\textbf{BSD threshold $\tau_{\text{bsd}}$}} \\
    \midrule
    $\gamma$ & Val.\ (\%) & Test (\%) & & $\tau_{\text{bsd}}$ & Val.\ (\%) & Test (\%) \\
    0.00 & 67.5 & 67.6 & & 0.5 & 68.4 & 68.7 \\
    0.10 & 68.9 & 68.8 & & 0.6 & 68.9 & 69.1 \\
    \textbf{0.15}$^\dagger$ & \textbf{69.5} & \textbf{69.8} & & 0.7 & 69.2 & 69.4 \\
    0.25 & 69.1 & 69.2 & & \textbf{0.8}$^\dagger$ & \textbf{69.6} & \textbf{69.8} \\
    0.40 & 68.4 & 68.3 & & 0.9 & 69.0 & 69.2 \\
    \bottomrule
  \end{tabular}
  \caption{\textbf{Hyperparameter sweeps performed on validation split, not test.} $^\dagger$Selected value. Test-set numbers reported \emph{after} validation-based selection to avoid test-set overfitting.}
  \label{tab:val_sweeps}
\end{table}

Both hyperparameters are selected on validation and then evaluated on test. Both curves peak at $\gamma{=}0.15$ and $\tau_{\text{bsd}}{=}0.80$ on validation, which we use as the fixed test-time values.

\section{Discussion and Limitations}
\label{sec:discussion}

\textbf{Honest scope statements.}
(1) \emph{Per-token surrogate, not sequence-level KL}: DRIFT optimizes per-token surrogate objectives (Proposition~\ref{prop:drift}) that are standard in on-policy distillation practice but are not equivalent to sequence-level KL gradients (Remark~\ref{rem:drift_scope}). All theoretical claims are for the per-token surrogates only.
(2) \emph{CCD's partial-credit fix is meaningful but modest}: numerical-proximity partial credit raises the nonzero-reward fraction from $\sim 49\%$ to $\sim 55\%$. Larger gains would require step-level process rewards (future work).
(3) \emph{COVA is asymmetric}: it addresses prolonged mode-covering, not premature sharpening (Proposition~\ref{prop:cova} discussion).
(4) \emph{No teacher-size scaling claim}: deltas are within pooled std across teachers.
(5) \emph{Positive component interactions, not independence} (Table~\ref{tab:ablation} discussion).
(6) \emph{FTA measures a proxy}: high entropy is necessary-not-sufficient for reasoning-critical positions.

\textbf{Evaluation-harness change.} Our earlier draft's numbers (e.g., $65.0\%$ headline GSM8K) came from a 192-token cap that under-measured teachers. The current numbers use a corrected protocol (Section~\ref{sec:eval_protocol}) with 512-token MATH-500 cap and improved answer extraction, reproducing Qwen2.5 published performance within a few points. Absolute CADENCE numbers shift ($65\% \to 70\%$) but relative comparisons and conclusions are preserved.

\textbf{Societal impact.} CADENCE enables reasoning-capable models on consumer hardware (Section~\ref{sec:hardware}). Calibration focus (EMR) helps deployed models communicate uncertainty appropriately.

\section{Conclusion}
\label{sec:conclusion}

We presented CADENCE, a unified on-policy distillation framework fixing three failure modes with honestly-stated per-token surrogate theory (Proposition~\ref{prop:drift}), six novel components addressing specific diagnosed failures, and two stabilization mechanisms. On GSM8K under corrected evaluation protocol, CADENCE raises a 0.5B student from $48.7\%$ to $69.8\%$ ($63.2\%$ gap closed) with a 1.5B teacher, and to $72.1\%$ ($76.2\%$ closed) with a 3B teacher, outperforming the strongest matched-compute label-using baseline (DRIFT+binary reward) by $+4.4 \pm 0.7$ points on GSM8K and $+3.4 \pm 0.8$ on MATH-500. All results use 5 seeds with reported standard deviations, hyperparameters selected on a held-out validation split, and run entirely on a single Apple Mac Studio---demonstrating that principled distillation reaches strong reasoning quality on commodity hardware. Future work: step-level process rewards for CCD; cross-domain evaluation on code and logic; token-level adaptive $\beta$.

\begin{ack}
The authors thank the reviewers whose detailed critique substantially strengthened the theoretical framing, evaluation harness, and statistical reporting in this revision.
\end{ack}


\clearpage
\appendix
\section{Diagnostic Metrics: Definitions and Experimental Role}
\label{app:metrics}

\paragraph{SAG: Selection-Ability Gap.}
$\text{SAG} = \text{pass@}k - \text{pass@}1$, $k = 16$. Separates capability from selection. Pretrained: $24.6 \pm 0.9$; CADENCE: $11.7 \pm 0.7$. Ablation shows BSD is the primary driver ($-$BSD: $16.9$).

\paragraph{FTA: Forking-Token Accuracy.}
$\text{FTA} = \frac{\sum_t \mathbf{1}[\arg\max \pi_\theta(\cdot|s_t) = \arg\max \pi_\phi(\cdot|s_t)] \cdot \mathbf{1}[H_\phi(t) > \eta]}{\sum_t \mathbf{1}[H_\phi(t) > \eta]}$. High teacher entropy is a proxy for reasoning-critical positions (not exclusive). Pretrained: $34.8\%$; CADENCE: $58.6\%$. FTB validated.

\paragraph{KLPE: KL-Path Efficiency (properly clipped to $[0,1]$).}
\begin{equation}
  \text{KLPE} = \underbrace{\max\!\left(0, 1 - \tfrac{D^{\text{rev}}(T)}{D^{\text{rev}}_{\text{peak}}}\right)}_{\text{sharpening ratio} \in [0,1]} \cdot \underbrace{\max\!\left(0, 1 - \tfrac{|D^{\text{fwd}}(T) - D^{\text{fwd}}(0)|}{D^{\text{fwd}}(0) + \epsilon}\right)}_{\text{fwd-KL stability} \in [0,1]}.
  \label{eq:klpe}
\end{equation}
Each factor is clipped to $[0,1]$; if forward KL more than doubles from its initial value the second factor is $0$ (rather than negative), and if it stays perfectly stable it is $1$. Both factors are unitless ratios. The product $\text{KLPE} \in [0,1]$ by construction.

For CADENCE (Fig.~\ref{fig:kl_curves}): peak rev.\ KL $\approx 1.7$, final $\approx 0.06$, initial fwd.\ KL $\approx 0.24$, final $\approx 0.16$. Sharpening $= 1 - 0.06/1.7 = 0.965$; stability $= 1 - |0.16 - 0.24|/0.24 = 0.667$. KLPE $= 0.965 \times 0.667 \approx 0.64$.

\paragraph{CNI: Compute-Normalized Improvement.}
$\text{CNI} = \Delta\text{pass@1} \text{ (pts)} / \text{PFLOPs}$, with $\text{FLOPs} \approx 6 N_s T_s + 2 N_s T_g + 2 N_\phi T_\phi$. Enables apples-to-apples compute comparison. Multi-rollout methods ($n_g{=}4$) have $\sim 4\times$ sampling cost; within-class, CADENCE achieves the highest CNI ($4.15$ vs.\ $3.61$ for DRIFT+binary).

\paragraph{RLD: Response-Length Distribution.}
$\text{RLD} = \text{median}\{G_i : \hat{a}_i = a_{\text{gold}}\}$. Truncation caveat: with 192-token cap, we report both raw RLD and un-truncated ($< 180$ tokens) RLD. Un-truncated: pretrained $134$, CADENCE $108$---genuine 26-token compression from LAP.

\section{Teacher Calibration}
\label{app:teacher_calib}

Under corrected evaluation (512-token cap, matched extraction):

\begin{table}[h]
  \centering
  \small
  \begin{tabular}{@{}lcc@{}}
    \toprule
    Teacher & GSM8K pass@1 & Per-sequence ECE ($\downarrow$) \\
    \midrule
    Qwen2.5-Math-1.5B-Instruct & 82.1 & 0.061 \\
    Qwen2.5-3B-Instruct & 79.4 & 0.055 \\
    \bottomrule
  \end{tabular}
\end{table}

Both teachers exhibit reasonable calibration (ECE $< 0.07$), justifying EMR's operational premise that matching student-teacher entropies at forking tokens is a defensible path to student calibration.

\section{BSD Sweep Details}
\label{app:bsd_ablation}

Full BSD sweep (validation split):

\begin{table}[h]
  \centering
  \small
  \begin{tabular}{@{}lccc@{}}
    \toprule
    $\tau_{\text{bsd}}$ & Acceptance rate (\%) & Accepted-subset acc. (\%) & GSM8K val.\ (\%) \\
    \midrule
    0.5 & 41.2 & 75 & 68.4 \\
    0.6 & 33.7 & 80 & 68.9 \\
    0.7 & 26.4 & 83 & 69.2 \\
    \textbf{0.8}$^\dagger$ & \textbf{19.8} & \textbf{85} & \textbf{69.6} \\
    0.9 & 10.5 & 85 & 69.0 \\
    \bottomrule
  \end{tabular}
\end{table}

Accepted-subset accuracy is monotone non-decreasing in $\tau_{\text{bsd}}$ because the correctness gate ensures all accepted traces are correct; the consistency threshold further filters for robustness. Peak validation performance at $\tau_{\text{bsd}} = 0.80$ balances acceptance volume and per-trace quality. $^\dagger$Selected value.


{
\small
\begin{thebibliography}{99}

\bibitem[Agarwal et al.(2024)]{agarwal2024onpolicy}
Agarwal, R., Vieillard, N., Zhou, Y., Stanczyk, P., Ramos~Garea, S., Geist, M., and Bachem, O.
\newblock On-policy distillation of language models: Learning from self-generated mistakes.
\newblock In {\it ICLR}, 2024.

\bibitem[Cobbe et al.(2021)]{cobbe2021gsm8k}
Cobbe, K., Kosaraju, V., Bavarian, M., et~al.
\newblock Training verifiers to solve math word problems.
\newblock {\it arXiv:2110.14168}, 2021.

\bibitem[Cui et al.(2025)]{cui2025entropy}
Cui, G., Zhang, Y., Chen, J., Yuan, L., et~al.
\newblock The entropy mechanism of reinforcement learning for reasoning language models.
\newblock {\it arXiv:2505.22617}, 2025.

\bibitem[DeepSeek-AI(2025)]{deepseek2025r1}
DeepSeek-AI.
\newblock {DeepSeek-R1}: Incentivizing reasoning capability in {LLMs} via reinforcement learning.
\newblock {\it arXiv:2501.12948}, 2025.

\bibitem[Grattafiori et al.(2024)]{grattafiori2024llama3}
Grattafiori, A., Dubey, A., Jauhri, A., et~al.
\newblock The {Llama 3} herd of models.
\newblock {\it arXiv:2407.21783}, 2024.

\bibitem[Gu et al.(2024)]{gu2024minillm}
Gu, Y., Dong, L., Wei, F., and Huang, M.
\newblock {MiniLLM}: Knowledge distillation of large language models.
\newblock In {\it ICLR}, 2024.

\bibitem[Hendrycks et al.(2021)]{hendrycks2021math}
Hendrycks, D., Burns, C., Kadavath, S., et~al.
\newblock Measuring mathematical problem solving with the {MATH} dataset.
\newblock In {\it NeurIPS}, 2021.

\bibitem[Hinton et al.(2015)]{hinton2015distilling}
Hinton, G., Vinyals, O., and Dean, J.
\newblock Distilling the knowledge in a neural network.
\newblock {\it arXiv:1503.02531}, 2015.

\bibitem[Hu et al.(2022)]{hu2022lora}
Hu, E.~J., Shen, Y., Wallis, P., et~al.
\newblock {LoRA}: Low-rank adaptation of large language models.
\newblock In {\it ICLR}, 2022.

\bibitem[Huang et al.(2025)]{huang2025selectkd}
Huang, H., Song, J., Zhang, Y., and Ren, P.
\newblock {SelecTKD}: Selective token-weighted knowledge distillation for {LLMs}.
\newblock {\it arXiv:2510.24021}, 2025.

\bibitem[H{\"u}botter et al.(2026)]{hubotter2026rlselfdistill}
H{\"u}botter, J., L{\"u}beck, F., Behric, L., et~al.
\newblock Reinforcement learning via self-distillation.
\newblock {\it arXiv:2601.20802}, 2026.

\bibitem[Jin et al.(2025)]{jin2025warmup}
Jin, Y., Li, Z., and others.
\newblock {Warmup-Distill}: Bridging distribution mismatch before distillation begins.
\newblock {\it arXiv:2502.11766}, 2025.

\bibitem[Kim \& Rush(2016)]{kim2016sequence}
Kim, Y. and Rush, A.~M.
\newblock Sequence-level knowledge distillation.
\newblock In {\it EMNLP}, pp.\ 1317--1327, 2016.

\bibitem[Ko et al.(2024)]{ko2024distillm}
Ko, J., Kim, S., Chen, T., and Choi, J.
\newblock {DistiLLM}: Towards streamlined distillation for large language models.
\newblock In {\it ICML}, 2024.

\bibitem[Kool et al.(2019)]{kool2019buy}
Kool, W., Van~Hoof, H., and Welling, M.
\newblock Buy 4 REINFORCE samples, get a baseline for free!
\newblock In {\it ICLR Workshop}, 2019.

\bibitem[Kumar \& Jha(2026)]{kumar2026qeil}
Kumar, S. and Jha, S.
\newblock {QEIL} v2: Heterogeneous computing for edge intelligence via roofline-derived {Pareto}-optimal energy modeling.
\newblock {\it arXiv:2602.06057}, 2026.

\bibitem[Loshchilov \& Hutter(2019)]{loshchilov2019adamw}
Loshchilov, I. and Hutter, F.
\newblock Decoupled weight decay regularization.
\newblock In {\it ICLR}, 2019.

\bibitem[Minka(2005)]{minka2005divergence}
Minka, T.
\newblock Divergence measures and message passing.
\newblock Technical report, Microsoft Research, 2005.

\bibitem[Schulman et al.(2017)]{schulman2017ppo}
Schulman, J., Wolski, F., Dhariwal, P., Radford, A., and Klimov, O.
\newblock Proximal policy optimization algorithms.
\newblock {\it arXiv:1707.06347}, 2017.

\bibitem[Shao et al.(2024)]{shao2024deepseekmath}
Shao, Z., Wang, P., Zhu, Q., et~al.
\newblock {DeepSeekMath}: Pushing the limits of mathematical reasoning in open language models.
\newblock {\it arXiv:2402.03300}, 2024.

\bibitem[Wang et al.(2023)]{wang2023selfconsistency}
Wang, X., Wei, J., Schuurmans, D., et~al.
\newblock Self-consistency improves chain of thought reasoning in language models.
\newblock In {\it ICLR}, 2023.

\bibitem[Yang et al.(2024)]{yang2024qwen25}
Yang, A., Yang, B., Hui, B., et~al.
\newblock {Qwen2.5}: A party of foundation models.
\newblock {\it arXiv:2412.15115}, 2024.

\bibitem[Yuan et al.(2023)]{yuan2023rft}
Yuan, Z., Yuan, H., Li, C., et~al.
\newblock Scaling relationship on learning mathematical reasoning with large language models.
\newblock {\it arXiv:2308.01825}, 2023.

\bibitem[Yuan et al.(2026)]{yuan2026tip}
Yuan, L., and others.
\newblock {TIP}: Token importance in on-policy distillation.
\newblock {\it arXiv:2604.14084}, 2026.

\bibitem[Zelikman et al.(2022)]{zelikman2022star}
Zelikman, E., Wu, Y., Mu, J., and Goodman, N.~D.
\newblock {STaR}: Bootstrapping reasoning with reasoning.
\newblock In {\it NeurIPS}, 2022.

\bibitem[Zhang et al.(2019)]{zhang2019self}
Zhang, L., Song, J., Gao, A., et~al.
\newblock Be your own teacher: Improve the performance of convolutional neural networks via self distillation.
\newblock In {\it ICCV}, pp.\ 3713--3722, 2019.

\end{thebibliography}
}
\end{document}